\documentclass[10 pt, conference]{ieeeconf} 

\IEEEoverridecommandlockouts 
\usepackage{url}
\usepackage{amsmath,amssymb}
\usepackage{float}
\usepackage{graphicx,subfigure}
\usepackage{color}
\usepackage[colorlinks=true,linkcolor=black]{hyperref}
\usepackage{breakurl}
\usepackage{cite}
\usepackage[ruled,linesnumbered]{algorithm2e}
\usepackage{todonotes}
\usepackage{etoolbox}
\usepackage{multirow}
\makeatletter
\patchcmd{\fs@ruled}
  {\def\@fs@pre{\hrule height.8pt depth0pt \kern2pt}}
  {\def\@fs@pre{\vspace*{5pt}\hrule height.8pt depth0pt \kern2pt}}
{}{}
\floatstyle{ruled}
\makeatother

\newtheorem{theorem}{Theorem}

\newcommand{\ie}{\emph{i.e.}}

\title{\LARGE \bf Tree Search Techniques for \\Minimizing Detectability and Maximizing Visibility}

\author{Zhongshun Zhang$^{1}$, Joseph Lee$^{2}$, Jonathon M. Smereka$^{2}$, Yoonchang Sung$^{1}$, Lifeng Zhou$^{1}$, and Pratap Tokekar$^{1}$%
\thanks{$^{1}$Zhongshun Zhang, Yoonchang Sung, Lifeng Zhou, and Pratap Tokekar are with the Department of Electrical \& Computer Engineering, Virginia Tech, USA {\tt\small \{zszhang, yooncs8, lfzhou,  tokekar\}}@vt.edu}%
\thanks{$^{2}$Joseph Lee and Jonathon M. Smereka are with U.S. Army TARDEC Ground Vehicle Robotics, Warren, MI 48397 USA {\tt\small \{joseph.s.lee34.civ, jonathon.m.smereka.civ\}@mail.mil}}%
}

\begin{document}

\maketitle
\thispagestyle{empty}
\pagestyle{empty}

\begin{abstract}
We introduce and study the problem of planning a trajectory for an agent to carry out a scouting mission while avoiding being detected by an adversarial guard. This introduces a multi-objective version of classical visibility-based target search and pursuit-evasion problem. In our formulation, the agent receives a positive reward for increasing its visibility (by exploring new regions) and a negative penalty every time it is detected by the guard. The objective is to find a finite-horizon path for the agent that balances the trade off between maximizing visibility and minimizing detectability. 

We model this problem as a discrete, sequential, two-player, zero-sum game. We use two types of game tree search algorithms to solve this problem: minimax search tree and Monte-Carlo search tree. Both search trees can yield the optimal policy but may require possibly exponential computational time and space. We propose several pruning techniques to reduce the computational cost while still preserving optimality guarantees. Simulation results show that the proposed strategy prunes approximately three orders of magnitude nodes as compared to the brute-force strategy. We also find that the Monte-Carlo search tree saves approximately one order of computational time as compared to the minimax search tree.
\end{abstract}
\section{Introduction}
%
Planning for visually covering an environment is a widely studied problem in robots with many real-world applications, such as environmental monitoring~\cite{tokekar2015visibility},  precision farming~\cite{peng2017view}, ship hull inspection~\cite{kim2015active}, and adversarial multi-agent tracking~\cite{hollinger2009efficient}. The goal is typically to find a path for an agent to maximize the area covered within a certain time budget or to minimize the time required to visually cover the entire environment. The latter is known as the Watchman Route Problem (WRP)~\cite{carlsson1999computing} and is closely related to the Art Gallery Problem (AGP)~\cite{o1987art}. The goal in AGP is to find the minimum number of cameras required to see all points in a polygonal environment. 
In this paper, we extend this class of visibility-based coverage problems to adversarial settings.

\begin{figure}
\centering{
\subfigure[Explore an environment.]{\includegraphics[width=0.47\columnwidth]{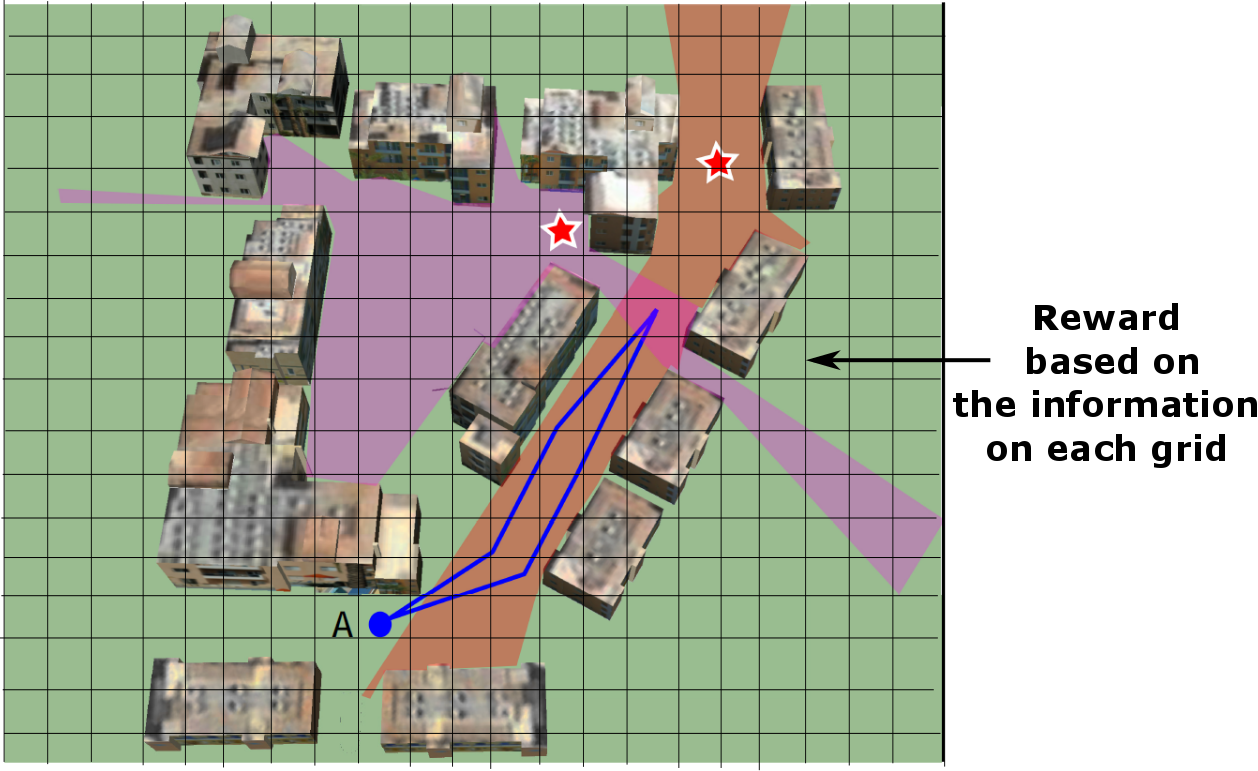}}
\subfigure[Reach a target point.]{\includegraphics[width=0.47\columnwidth]{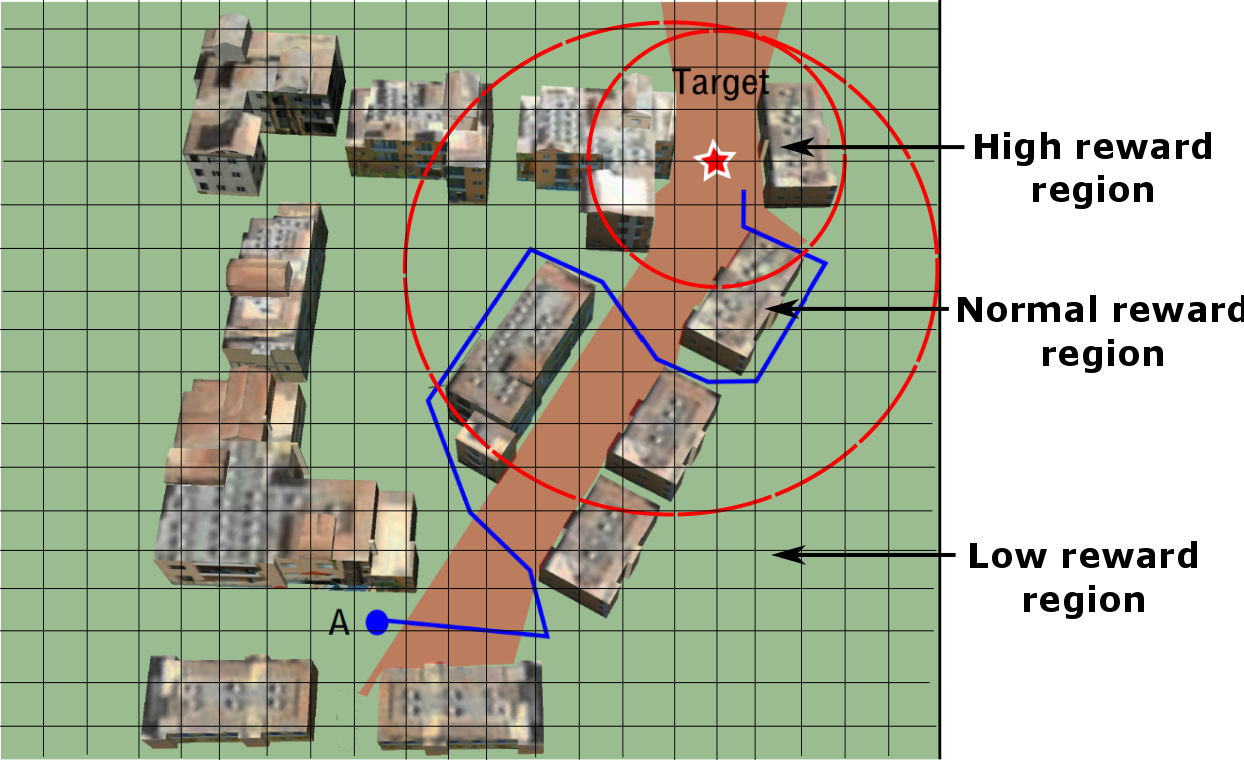}}
}
\caption{Two example missions. Maximizing visibility implies maximizing the total reward collected along a finite horizon path while minimizing detectability can be achieved by avoiding the grid cells from where the agent can be seen. Both types of mission can be formulated by assigning different reward functions over a grid-based map.\label{Simulation_environment}      
}
\end{figure}

We consider scenarios where the environment also contains a guard that is actively (and adversarially) searching for the agent (Figure~\ref{Simulation_environment}). The agent, on the other hand, is tasked with covering the environment while avoiding detection by the guard. This models stealth reconnaissance missions. We consider the version where there is a finite time horizon within which the agent must complete its mission. The objective of the agent is to maximize the total area covered within the given horizon while at the same time minimize the number of times it is detected by the guard.


We adopt a game-theoretic approach for this problem where the agent maximizes the total reward collected and the guard minimizes the total reward. The total reward is a weighted combination of positive and negative rewards. The positive reward depends on the specific task at hand. For example, when the task is to scout an environment (Figure~\ref{Simulation_environment}(a)), the positive reward can be the total area that is scanned by the agent along its path. When the task is to reach a goal position (Figure~\ref{Simulation_environment}(b)), the positive reward can be the function of the distance to the goal. 
The agent receives a negative reward whenever it is detected by the guard. The negative reward can also be defined based on the specific application. In this paper, we consider the case where the agent receives a fixed negative reward every time it is detected by the agent. However, other models (e.g., time-varying negative rewards) can be easily incorporated. The total reward is a combination of the two reward functions.

This problem is a new variant of the classical pursuit-evasion problems~\cite{zhang2016non}. 
Our approach is closer to the visibility-based~\cite{stiffler2017complete, MaBeMuBeHu18} pursuit-evasion games.  However, the main distinction is that in classical pursuit-evasion games, the goal of the evader (\ie, the agent in our setting) is to always evade the pursuer (\ie, the guard) whereas in our setting, the agent has to explore the environment to increase its visibility while at the same time staying away from the guard. 

Broadly speaking, the proposed problem is a combination of classical pursuit-evasion games and visibility-based routing such as the WRP (where the objective is to minimize the time required to observe the environment)~\cite{carlsson1999finding}. Here, we combine the two tasks albeit for discrete environments. Furthermore, the definition of winning a game is different from the conventional pursuit-evasion games. In general, the pursuer wins the game if the distance between the pursuer and evader becomes less than a threshold~\cite{bopardikar2007sensing}, or if the evader is surrounded by the pursuer~\cite{jin2011heuristic}. In our case, however, the goal is not only to capture or to avoid detection by the other player, but also to maximize the area explored, which has not been considered in conventional pursuit-evasion games.

We abstract the underlying geometry and model the problem as a discrete, sequential, two-player, zero-sum game.
Minimax tree search~\cite{gelly2006exploration} and Monte-Carlo Tree Search (MCTS)~\cite{russell2009artificial} are well-known algorithms to solve discrete, two-player, zero-sum games.
Both techniques build a search tree that contains all possible (or a subset of all possible) actions for both players over planning horizons. In general, the search tree will have size that is exponential in the planning horizon. Pruning techniques, such as alpha-beta pruning~\cite{russell2016artificial}, can be employed in order to prune away branches that are guaranteed not to  be part of the optimal policy.
We propose additional pruning techniques (Theorems~\ref{theorem:1}-\ref{theorem:Compare_past}) using the structural properties of the underlying problem to further reduce the computational expense for both the minimax tree search and MCTS. We guarantee that the pruned search tree still contains the optimal policy for the agent.

The contributions of this paper are as follows: (1) We introduce a new problem of minimizing detectability and maximizing visibility as a sequential, two-player, zero-sum game between an agent and a guard; (2) We propose pruning strategies that exploit the characteristics of the proposed problem and that can be applied for both minimax search tree and Monte-Carlo search tree. 


The rest of the paper is organized as follows. We begin by describing the problem setup in Section~\ref{sec:probform}. We present the two tree search techniques in Section~\ref{sec:game} and present the pruning techniques in Section~\ref{sec:pruning}. The simulation results are presented in Section~\ref{sec:sim}. Section~\ref{sec:con} summarizes the paper and outlines future work. 

%
\section{Problem Formulation}~\label{sec:probform}
We consider a grid-based environment where each cell within the environment is associated with a positive reward. Our approach is to formulate the proposed problem by appropriately designing the reward function --- the agent obtains positive rewards for maximizing visibility (depending on the type of missions) and receives negative rewards when detected by the guard. The reward is used to measure both the detectability of a guard and the visibility of an agent.

In an exploration mission, the positive reward can be a function of the number of previously unseen cells visible from the current agent position (Figure~\ref{Simulation_environment}-(a)). In a mission where the objective is to reach a goal position, the positive reward can be defined as a function of the (inverse of the) distance to the guard (Figure~\ref{Simulation_environment}-(b)). The agent receives a negative reward when it is detected by the guard (\ie, when it moves to the same cell as the guard or to a cell that lies with the guard's visibility region). At every turn (\ie, the time step), both the agent and the guard can move to one of their neighboring cells (\ie, the action).


\begin{figure}
\centering{
\subfigure[The case when the agent is detected by the guard.]{\includegraphics[width=0.32\columnwidth]{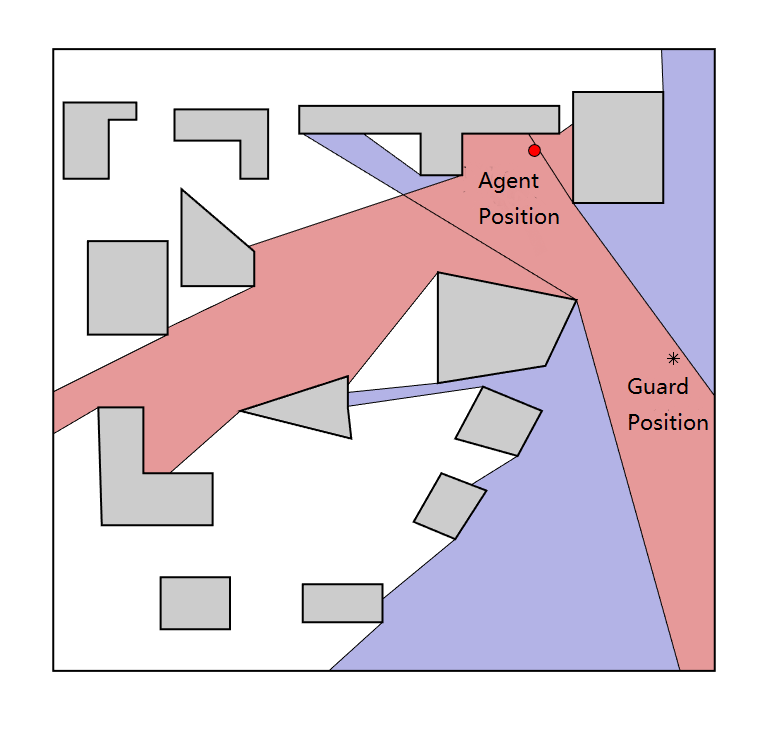}}
\subfigure[The case when the agent is not detected by the guard.]{\includegraphics[width=0.33\columnwidth]{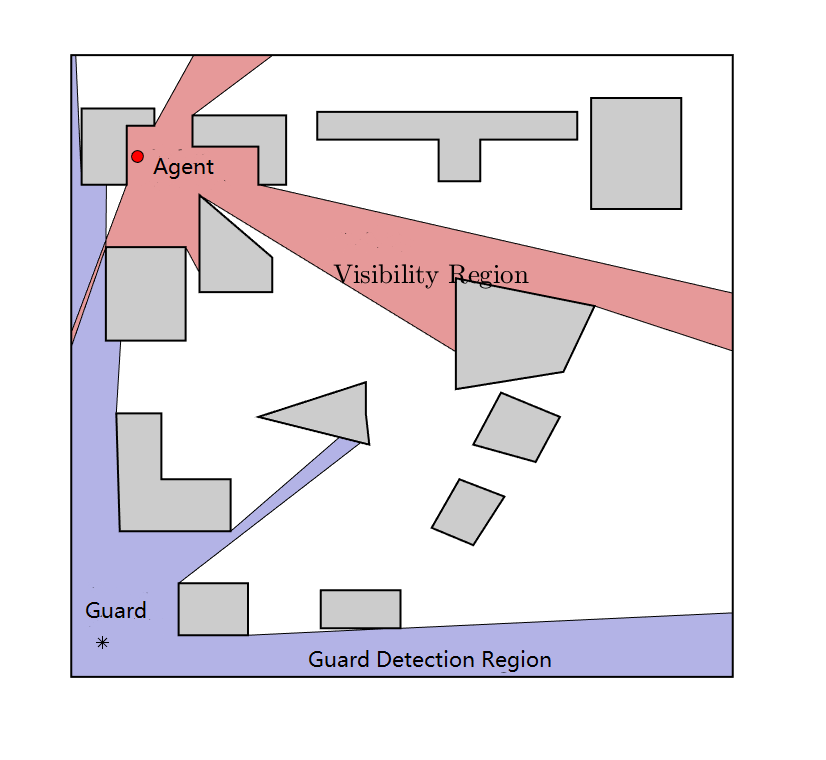}}
\subfigure[The agent and the guard move in a grid-based environment.]{\includegraphics[width=0.3\columnwidth]{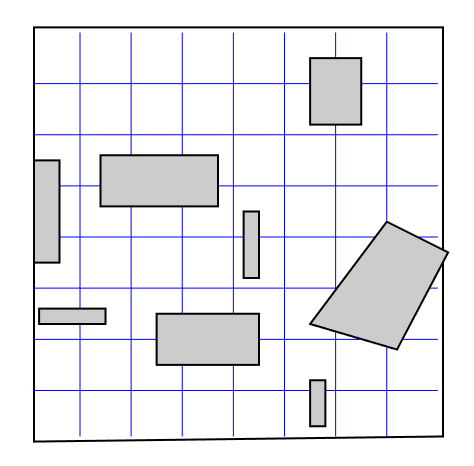}}
}
\caption{A negative penalty will be added if the agent is inside the guard's visibility polygon (\ie, the blue). In a reconnaissance mission, the area of the agent's visibility polygon (\ie, the red) is considered as a positive reward. Both the agent and the guard move in the grid-based environment, as in (c). }
\label{Detected_or_not}         
\end{figure}

We make the following assumptions:
(1) The agent and the guard move in the same grid-based map and can move one edge in one time step.
(2) Both the agent and the guard know the full grid-based map \emph{a priori}.
(3) We assume that the agent and the guard have known sensing ranges (not necessarily the same). In this paper, we assume that both sensing ranges are unlimited for ease of illustration, however, the case of limited sensing range can easily be incorporated. 
(4) The guard has a sensor that can detect the agent when the agent is within its visibility region.
(5) There is no motion uncertainty associated with the agent and guard actions.
(6) The agent is aware of initial position of the guard.

While the last assumption may seem restrictive, there are practical scenarios where it is justified. For example, Bhadauria and Isler~\cite{bhadauria2011capturing} describe a visibility-based pursuit-evasion game where police helicopters can always provide the global positions of the evader to the pursuer that is moving on the ground and may not be able to directly see the pursuer. Thus, even if the guard is not in the field-of-view of the agent, the agent may still know the actual position of the guard by communicating with other (aerial) agents. Note that the agent still does not know where the guard will move next, thereby, making the problem challenging.

In general, the environment could be any discrete environment, not just grids, as long as they satisfy the above requirements. 
In fact, continuous environments can be appropriately discretized such that they satisfy the above assumptions. Medial-axis transformation or skeletonization~\cite{volkov2011environment} and randomized methods such as probabilistic roadmaps~\cite{kavraki1996analysis} and Rapidly-exploring Random Trees (RRTs)~\cite{lavalle2006planning} are common environment discretization techniques. RRTs have been used to solve two-player, zero-sum pursuit evasion games~\cite{karaman2010incremental}. Any such suitable technique can be used. The complexity of the tree search algorithm  will depend on the number of vertices (or grid cells) in a given discretization.

The agent's objective can be written as:
\begin{equation}
\max_{\pi_a(t)} \min_{\pi_g(t)} \left\{ R(\pi_a(t))  - \eta(\pi_a(t),\pi_g(t)) P \right\}.
\label{ob_function}
\end{equation}
On the other hand, the objective of the guard is:
\begin{equation}
 \min_{\pi_g(t)} \max_{\pi_a(t)} \left\{ R(\pi_a(t))  - \eta(\pi_a(t),\pi_g(t)) P \right\},
\end{equation}
where 
$\pi_a(t)$ denotes an agent's path from time step $0$ to $t$. 
$\pi_g(t)$ denotes a guard's path from time step $0$ to $t$. 
$R(\pi_a(t))$ denotes the positive reward collected by the agent along the path from time step $0$ to $t$.  
$P$ is a constant which gives the negative reward for the agent whenever it is detected by the guard.
 $\eta(\pi_a(t),\pi_g(t)) $ indicates the total number of times that the agent is detected from time step $0$ to $t$.
For the rest of the paper, we model $R(\pi_a(t))$ to be the total area that is visible from the agent's path $\pi_a(t)$.



We model this as a discrete, sequential, two-player zero-sum game between the guard and the agent. In the next section, we demonstrate how to find the optimal strategy for this game and explain our proposed pruning methods.

\section{Solution: Two Types of Search Trees }~\label{sec:game}

We refer the agent and the guard as MAX and MIN players, respectively. Even though the agent and the guard move simultaneously, we can model this problem as a turn-based game. At each time step, the agent moves first to maximize the total reward, and then the guard moves to minimize the total reward. This repeats for a total of $T$ planning steps. In this section, we first show how to build a minimax search tree to find the optimal policy. Then, we show how to construct a Monte-Carlo search tree to solve the same problem. The advantage of MCTS is that it finds the optimal policy in lesser computational time than minimax tree --- a finding we corroborate in Section~\ref{sec:sim}.

\subsection{Minimax Tree Search}

A minimax tree search is a commonly used technique for solving two-player zero-sum games~\cite{russell2016artificial}. Each node stores the position of the agent, the position of the guard, the polygon that is visible to the agent along the path from the root node until the current node, and the number of times the guard detects the agent along the path from the root node to the current node. The tree consists of the following types of nodes:
\begin{itemize}
\item \emph{Root node}: The root node contains the initial positions of the agent and the guard. 
\item \emph{MAX level}: The MAX (i.e., agent) level expands the tree by creating a new branch for each neighbor of the agent's position in its parent node from the previous level (which can be either the root node or a MIN level node). The agent's position and its visibility region are updated at each level. The guard's position and the number of times the agent is detected are not updated at this level.
\item \emph{MIN level}: The MIN (i.e., guard) level expands the tree by creating a new branch for each neighbor of the guard's position in its parent node (which is always a MAX level node). The guard's position is updated at each level. The total reward is recalculated at this level based on the agent's and guard's current visibility polygons and the total number of times the agent is detected up to the current level.
\item \emph{Terminal node}: The terminal node is always a MIN level node. When the minimax tree is fully generated (\ie, the agent reaches a finite planning horizon), the reward value of the terminal node can be computed. 
\end{itemize}
The reward values are backpropagated from the terminal node to the root node. The minimax policy chooses an action which maximizes and minimizes the backpropagated reward at the MAX and the MIN nodes, respectively.

\begin{figure}
\centering
\includegraphics[width=0.65\columnwidth]{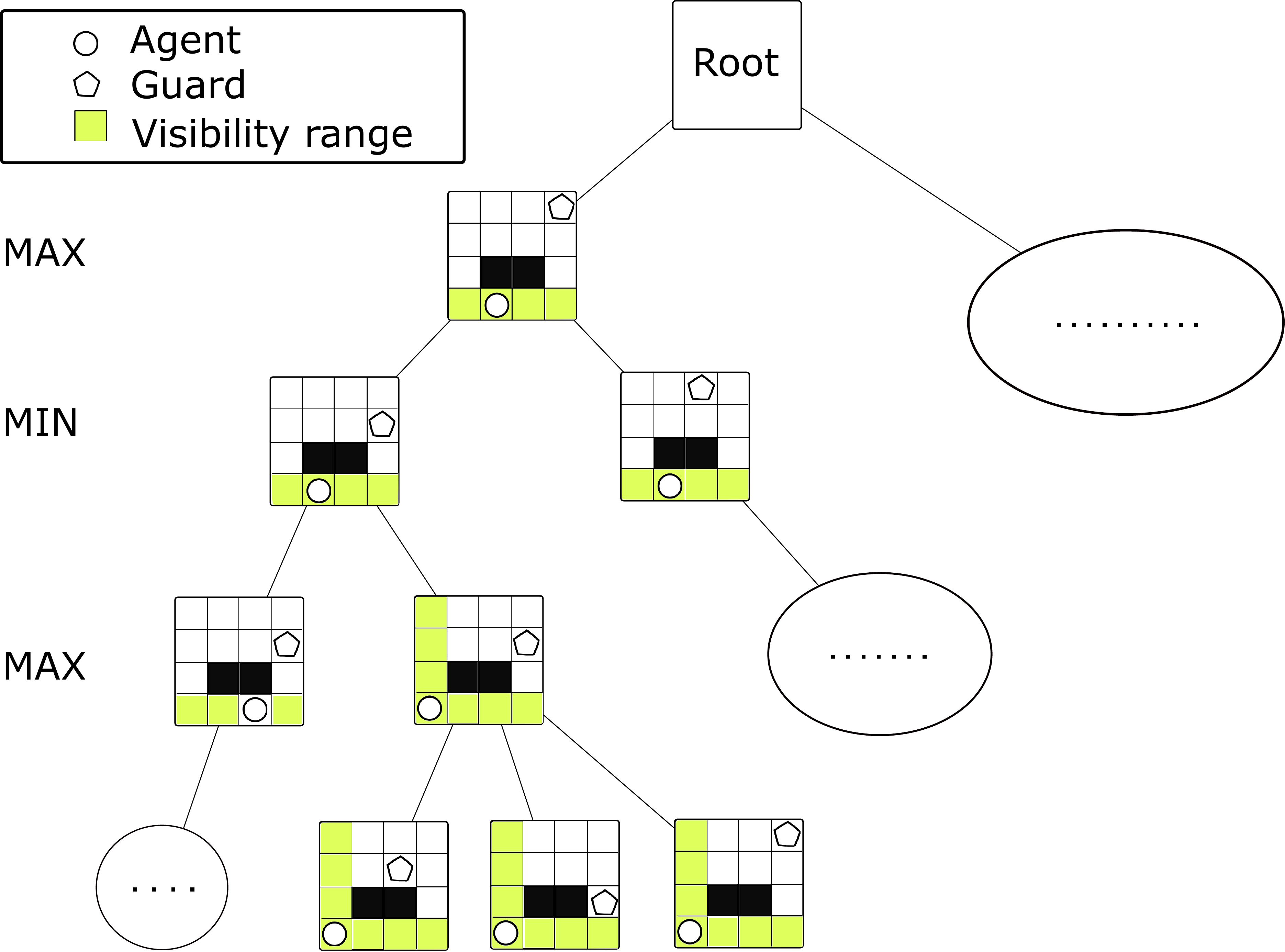}
\caption{A (partial) minimax game tree. The root node contains the initial states of the agent and the guard. Two successive levels of the tree correspond to one time step. The agent moves first to an available position in order to maximize the reward (MAX level). The guard moves subsequently to a neighboring cell to minimize the agent's reward (MIN level).}
\label{tree}
\end{figure}

Figure~\ref{tree} illustrates the steps to build a minimax tree that yields an optimal strategy by enumerating all possible actions for both the agent and the guard. Algorithm~1 presents the algorithm of minimax tree search.

 \begin{algorithm}[h]
 \SetAlgoLined
 \SetKw{Break}{break}
 \SetKwProg{Fn}{function}{ }{end}
 \Fn{$\operatorname{Minimax}(\mathit{node}, \mathit{depth}, \alpha, \beta, \mathit{state})$}
 {
   \uIf{$\mathit{node}$ is a terminal node} {
       \Return{$\mathit{value}$}
   }
   \uElseIf{$\mathit{state}$ is at the agent level}{
       \For{each child $v$ of $\mathit{node}$}
       {
          $V \leftarrow \operatorname{Minimax}(v, \mathit{depth}-1, \alpha, \beta, \mathrm{MIN})$
         
          $\mathit{bestvalue} \leftarrow \max(\mathit{bestvalue}, V)$
         
          $\alpha \leftarrow \max(bestvalue, \alpha)$
         
          \tcp{Alpha-beta pruning}
         
          \If{$\beta \leq \alpha$}{\Break}
          \tcp{Proposed condition}
          \If{pruning condition is true}{\Break}
          \Return{$\mathit{value}$}
       }
   }
   \Else{
       \For{each child v of node}{
          $V \leftarrow \operatorname{Minimax}(v, depth-1, \alpha, \beta, \mathrm{MAX})$\\
          $\mathit{bestvalue} \leftarrow \min(\mathit{bestvalue}, V)$\\
          $\beta \leftarrow \min(\mathit{bestvalue}, \beta)$\\
          \If{$\beta \leq \alpha$}{\Break}
          \If{pruning condition is true}{\Break}
          \Return{$\mathit{value}$}
       }
   }
   $\text{Initial} \leftarrow \left\{S_0\right\}$,Map \\
 $A_r(s),A_t(s) \leftarrow \operatorname{Minimax}(S_0,1,-\infty,\infty,\mathrm{MAX})$\\
 }
 \label{alg:minimax_pruning}
 \caption{ The minimax tree search}
 \end{algorithm}



\subsection{Monte-Carlo Tree Search}
In the naive minimax tree search, the tree is expanded by considering all the neighbors of a leaf node, one-by-one. In MCTS, the tree is expanded by carefully selecting one of the nodes to expand. Which node to select for expansion depends on the current estimate of the value of the node. The value is found by simulating many \emph{rollouts}. In each rollout, we simulate one instance of the game, starting from the selected node, by applying some arbitrary policy for the agent and the guard till the end of the planning horizon, $T$. The total reward collected is stored at the corresponding node. This reward is then used to determine how likely is the node to be chosen for expansion in future iterations. 

\begin{figure}[H]
\centering{
\includegraphics[width=0.99\columnwidth]{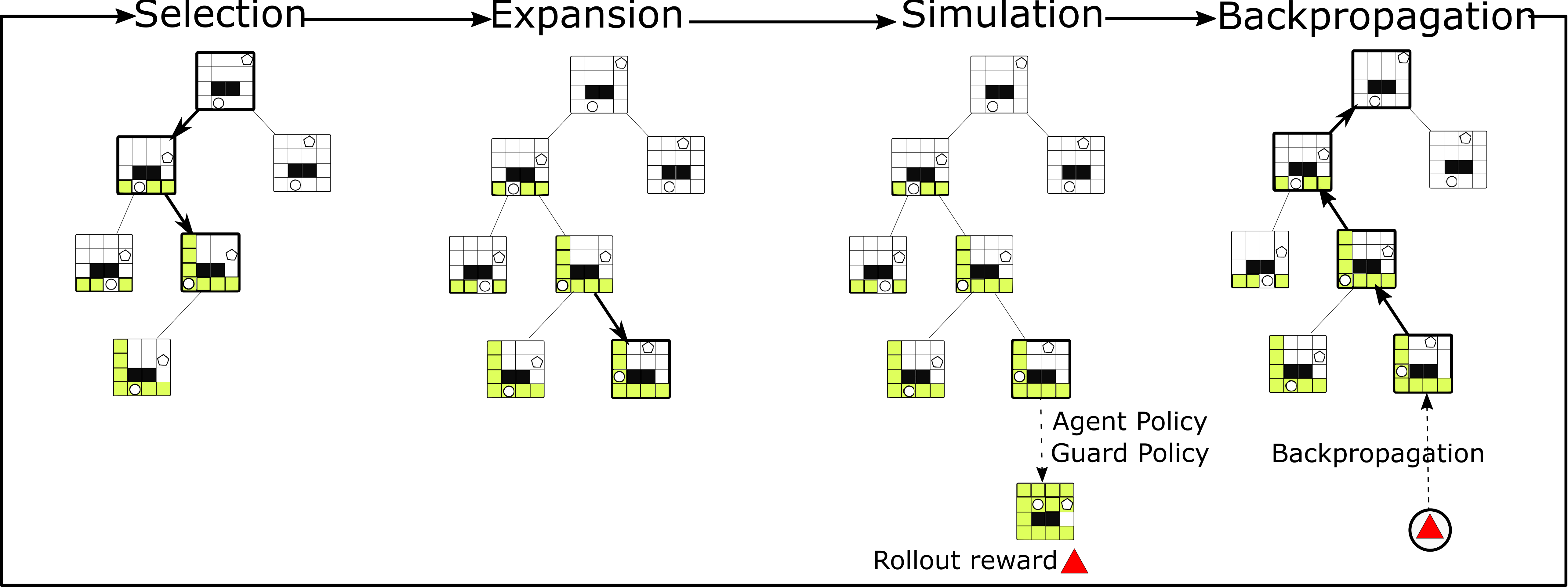}
\caption{Monte-Carlo search tree iteration steps.}
\label{MCST_Simulation_environment} }    
\end{figure}

Each node in the Monte-Carlo search tree stores the total reward value, $Q(v)$, and the number of times the node is visited, $N(v)$. Each iteration of MCTS consists of the following four steps~\cite{chaslot2008monte} (Figure ~\ref{MCST_Simulation_environment}):
\begin{itemize}
\item \textbf{Selection} (Line~\ref{MCTS:line:Selection1} in Algorithm~2): Starting from the root node (in every iteration), the node selection algorithm uses the current reward value to recursively descend through the tree until we reach a node that is not at the terminal level (\ie, corresponding to time $T$) and has children that have never been visited before. We use the Upper Confidence Bound (UCB) to determine which node should be selected (Lines \ref{MCTS:line:UCB_max}--\ref{MCTS:line:UCB_min} in  Algorithm~2). The UCB value takes into account not only the average of the rollout reward obtained but also the number of times the node has been visited. If a node is not visited often, then the second term in the UCB value will be high, improving its likelihood of getting selected. The constant $c$ in Lines \ref{MCTS:line:UCB_max}--\ref{MCTS:line:UCB_min} in  Algorithm~2 is the weighting parameter. At the agent level, we choose the node with the highest UCB value while at the guard level with the lowest UCB value.
\item \textbf{Expansion} (Lines \ref{MCTS:line:Expand1}--\ref{MCTS:line:Expand2} in Algorithm~2): Child nodes (one or more) are added to the selected nodes to expand the tree. If the child node is at the agent level, the node denotes one of the available actions for the agent. If the child node is at the guard level, the node denotes one of the available actions for the guard. Expansion details are given in Algorithm 3.
\item \textbf{Rollout} (Line \ref{MCTS:line:rollout} in Algorithm~2): A Monte-Carlo simulation is carried out from the expanded node for the remaining planning horizon. The agent and the guard follow a uniformly at random policy. Based on this, the total reward for this simulation is calculated. Rollout details are given in Algorithm~4.
\item \textbf{Backpropagation} (Lines \ref{MCTS:line:Backpropagation1}--\ref{MCTS:line:Backpropagation2} in Algorithm~2): The total reward found is then used to update the reward value stored at each of the predecessor nodes.
\end{itemize}


\begin{algorithm}[h]
\SetAlgoLined
\SetKwProg{Fn}{function}{ }{end}
\Fn{$\operatorname{MCTS}(\textit{Tree}, \textit{Initial agent and guard state})$}
{
   Create root node $v_0$ with initial guard and agent state $s_0$\;
   {
   \While{maximum number of iterations not reached}
   {
      \tcp{Selection} 
      $v_i \leftarrow \operatorname{Monte\_Carlo\_Selection}(\textit{Tree},v_0)$ 		       \label{MCTS:line:Selection1}
      
      \tcp{Expand or rollout}

      \eIf{$\operatorname{level}(v_i) = T$ {\bf and} $N(v_i) = 0$ }{
         \tcp{Expand}
         
         $Tree \leftarrow \operatorname{Expand}(Tree,v_i)$
         \label{MCTS:line:Expand1}
         
         \If{Newly added node can be pruned}{\textbf{break}}
         \label{MCTS:line:Expand2}
      }{
         \tcp{Rollout}
         $R \leftarrow \operatorname{Rollout}(v_i)$\;
         \label{MCTS:line:rollout}
      }
      \tcp{Backpropagation}\
     \label{MCTS:line:Backpropagation1} 
}
}
{return $\textit{Tree}$}
}
\caption{Monte-Carlo Tree Search with Pruning}
\label{MCTS} 
\end{algorithm}

\begin{algorithm}[h]
  \SetAlgoLined
\SetKwProg{Fn}{function}{ }{end}
\Fn{$\operatorname{Monte\_Carlo\_Selection}(Tree, v_i)$}
{
 \While{$\operatorname{level}(v_i) \neq \mathrm{TERMINAL}$}{
  \eIf{$\operatorname{level}(v_i) = \mathrm{AGENT}$}{
     $v_i\leftarrow \underset{v'\in \operatorname{children}(v_i)}{\operatorname{arg\,max}} \frac{Q(v')}{N(v')}+c\sqrt{\frac{2\ln N(v')}{N(v')}}$ \label{MCTS:line:UCB_max}
  }
  {
  $v_i\leftarrow \underset{v'\in \operatorname{children}(v_i)}{\operatorname{arg\,min}} \frac{Q(v')}{N(v')}+c\sqrt{\frac{2\ln N(v')}{N(v')}}$ \label{MCTS:line:UCB_min} 
  }
 }
}

\caption{ MCTS selection}
\label{MCTS_selection} 
\end{algorithm}

\begin{algorithm}[h]
  \SetAlgoLined

\SetKwProg{Fn}{function}{ }{end}
\Fn{$\operatorname{Rollout}(v)$}
{
   $R\leftarrow 0$\\
   \While{$\operatorname{level}(v) \neq 2T+1$}
   {
      \eIf{$\operatorname{level}(v) = \mathrm{AGENT}$}
      {
         $v \leftarrow$ choose an agent action at random
      }
      {
         $v \leftarrow$ choose a guard action at random \\
         $R \leftarrow$ update reward
      }
      return $R$
   }
}

\caption{MCTS rollout}
\label{MCTS_rollout} 
\end{algorithm}

Given a sufficient number of iterations, the MCTS with UCB is guaranteed to converge to the optimal policy~\cite{baier2013monte}. However, this may still require building an exponentially sized tree. In the next section, we present a number of pruning strategies to reduce the size of the tree. In Section~\ref{sec:sim}, we also evaluate the effect of the number of iterations on the solution quality.


\section{Pruning Techniques} \label{sec:pruning}

In this section, we present several pruning techniques to reduce the size of the tree and the computational time required to build the minimax and the MCTS. Pruning a node implies that the node will never be expanded (in both types of trees). 
In MCTS, if a node is pruned we simply will break to the next iteration of the search. Pruning the tree results in considerable computational savings which we quantify in Section~\ref{sec:sim}.

In the case of the minimax search tree, we can apply a classical pruning strategy called \emph{alpha-beta pruning}~\cite{russell2009artificial}.
Alpha-beta pruning maintains the minimax values at each node by exploring the tree in a depth-first fashion. It then prunes nodes, if a node is clearly dominated by another. See the textbook by Russell and Norvig~\cite{russell2009artificial} for more details. Alpha-beta pruning is preferable when the tree is built in a depth first fashion. However, we can exploit structural properties of this problem to further prune away nodes without needed to explore a subtree fully. We propose  strategies that find and prune redundant nodes before the terminal level is reached.



 

Our proposed pruning techniques apply for both types of trees. Therefore, in the following we refer to a ``search tree'' instead of specifying whether it is minimax or MCTS.

Our first proposed class of pruning techniques (\ie, Theorems~\ref{theorem:1} and~\ref{theorem:2}) is based on the properties of the given map. Consider the MIN level and the MAX level separately. The main idea of these pruning strategies is to compare two nodes $A$ and $B$ at the same level of the tree, say the MAX level. In the worst case, the node $A$ would obtain no future positive reward while always being detected at each time step of the rest of the horizon. Likewise, in the best case, the node $B$ would collect all the remaining positive reward and never be detected in the future. If the worst-case outcome for node $A$ is still better than the best-case outcome for node $B$, then node $B$ will never be a part of the optimal path. It can thus be pruned away from the search tree. Consequently, we can save time that would be otherwise spent computing all of its successors.  Note that these conditions can be checked even before reaching the terminal node of the subtrees at $A$ or $B$. 

Given a node in the search tree, we denote the remaining positive reward (unscanned region) for this node by $F(\cdot)$. Note that we do not need to know $F(\cdot)$ exactly. Instead, we just need an upper bound on $F(\cdot)$. This can be easily computed since we know the entire map information \emph{a priori}. The total reward collected by the node $A$ and by the node $B$ from time step $0$ to $t$ are denoted by $R^A(t)$ and $R^B(t)$, respectively.

\begin{theorem}
Given a time horizon $T$, let $A$ and $B$ be two sibling nodes in the same MAX level of the search tree at time step $t$. If 
$R^A(t) - (T-t)\eta \ge R^B(t) + F(B)$, then the node $B$ can be pruned without loss of optimality.
 \label{theorem:1}
\end{theorem}
\begin{proof}
In the case of the node $A$, the worst case occurs when in the following $T-t$ steps the agent is always detected at every remaining step and collects zero additional positive rewards. After reaching the terminal tree level, the reward backpropagated to node $A$ will be $R^A(t) - (T-t)\eta$. 
For the node $B$, the best case occurs in the following $T-t$ steps when the agent is never detected but obtains all remaining positive rewards. In the terminal tree level, the node $B$ collects the reward of $ R^B(t) + F(B)$.

Since $R^A(t) - (T-t)\eta \ge R^B(t) + F(B)$ and both nodes are at the MAX level, it implies that the reward returned to the node $A$ is always greater than that returned to the node $B$. Therefore, the node $B$ will not be a part of the optimal policy and can be pruned without affecting the optimality.
\end{proof}

Similarly, consider that the node $A$ and the node $B$ are located in the MIN level. The same idea of Theorem~\ref{theorem:1} holds as follows.

\begin{theorem}
Given a time horizon $T$, let $A$ and $B$ be two sibling nodes in the same MIN level of the search tree at time step $t$. If $R^A(t) + F(A)\le R^B(t)-(T-t)\eta $, then the node $B$ can be pruned without loss of optimality.
 \label{theorem:2}
\end{theorem}
The proof of Theorem~\ref{theorem:2} is similar to that of Theorem~\ref{theorem:1}.
 
The main idea of the second type of pruning strategy (\ie, Theorem~\ref{theorem:Compare_past}) comes from the past path (or history). If two different nodes have the same agent and guard position but one node has a better history than the other, then the other node can be pruned away.

Here, we denote by $S^A(\pi(t))$ and $S^B(\pi(t))$ the total scanned region in the node $A$ and the node $B$ from time step $0$ to $t$, respectively.

\begin{theorem}
Given a time horizon $T$ and $0<t_1<t_2<T$, let the node $A$ be at the level $t_1$ and the node $B$ be at the level $t_2$, respectively such that both nodes are at a MAX level. If (1) the guard's position stored in the nodes $A$ and $B$ are the same, (2) $S^A(\pi(t_1)) \supset S^B(\pi(t_2))$, and (3) $R^A(t) > R^B(t)+(t_2-t_1)\eta$, then the node $B$ can be pruned without loss of optimality.
 \label{theorem:Compare_past}
\end{theorem}
\begin{proof}
With $0<t_1<t_2<T$, we have the node $B$ appear further down the tree as compared to the node $A$. $S^A(\pi(t_1)) \subseteq S^B(\pi(t_2))$ indicates that the node $A$'s scanned area is a subset of the node $B$'s scanned area.

Since the nodes $A$ and $B$ contain the same guard and agent positions, one of the successors of node $A$ contains the same guard and agent positions as node $B$. Since $R^A(t) \geq R^B(t)+(t_2-t_1)\eta$ and $S^A(\pi(t_1)) \supset S^B(\pi(t_2))$, the value backpropagated from the successor of node $A$ will always be greater than the value backpropagated from the path of node $B$. Furthermore,  more reward can possibly be collected by node $A$ since $S^A(\pi(t_1)) \subseteq S^B(\pi(t_2))$. Thus, the node $B$ will never be a part of the optimal path and can then be pruned away. 
\end{proof}

\section{Evaluation} \label{sec:sim}
In this section, we evaluate the proposed techniques in the context of a reconnaissance mission. We assume the visibility range of the agent and the guard are both unlimited (only restricted by the obstacles in the environment). We use the VisiLibity library~\cite{VisiLibity} to compute the visibility polygon. The simulation is executed in MATLAB.

First, we present two qualitative examples that show the path found by the minimax algorithm. Second, we compare the computational cost of the two search tree algorithms with and without pruning. Third, we study the trade-off between solution quality and computational time by varying the parameters in MCTS. Finally, we show how to apply the tree search technique in an online fashion.

\subsection{Qualitative Examples}

Figures~\ref{High_minimax} and \ref{Low_minimax} show two examples of the policy found by Monte-Carlo tree search method, using high and low negative penalty values ($P$ in Equation~\ref{ob_function}) respectively. 
Both the minimax tree search and MCTS can find the same optimal solution for these instances.
We use a $25 \times 15$ grid environment. With higher negative reward $P=30$, the agent tends to prefer avoiding detection by the guard (Figure~\ref{High_minimax}). With a lower negative reward $P=3$, the agent prefers to explore more area (Figure~\ref{Low_minimax}).

\begin{figure*}[thb]
\centering
\includegraphics[width=0.16\textwidth]{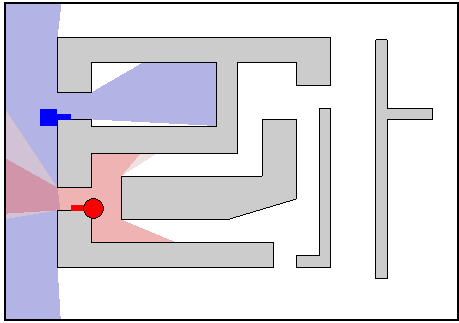}
\includegraphics[width=0.16\textwidth]{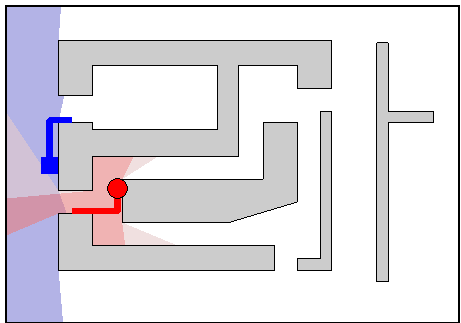}
\includegraphics[width=0.16\textwidth]{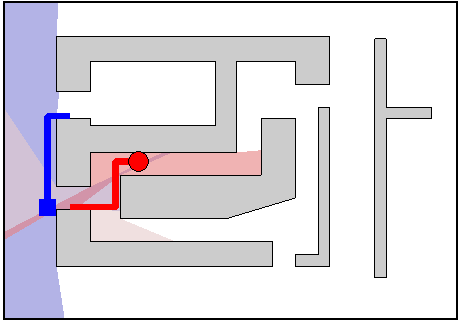}
\includegraphics[width=0.16\textwidth]{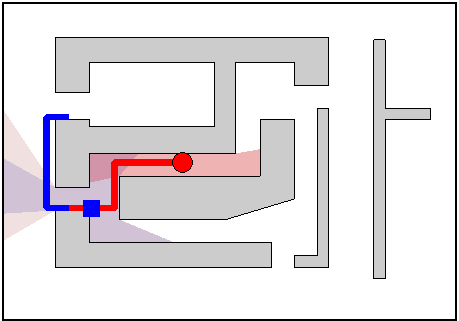}
\includegraphics[width=0.16\textwidth]{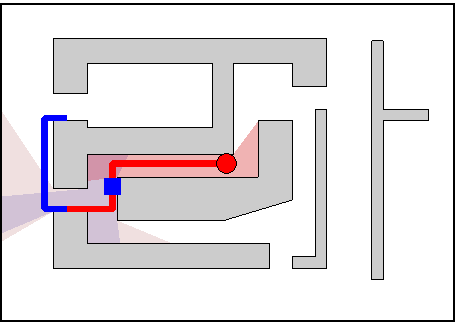}
\caption{Qualitative example (higher penalty $P=30$): Path for the agent (red) and the guard (blue) is given by MCTS for $T=10$. The environment is a $20 \times 15$ grid. With a higher penalty, the agent prefers paths where it can hide from the guard at the expense of the area explored (from left to right, $t=2,4,6,8,10$.). Figure~\ref{Low_minimax} shows the case with a lower penalty. 
}  
\label{High_minimax}         
\end{figure*}

\begin{figure*}[thb]
\centering
\includegraphics[width=0.16\textwidth]{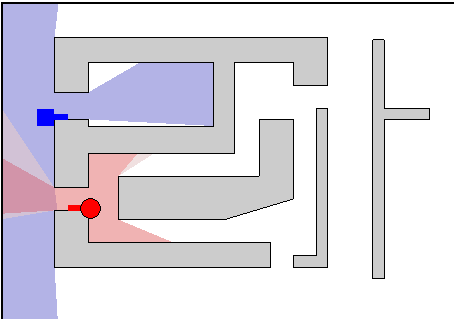}
\includegraphics[width=0.16\textwidth]{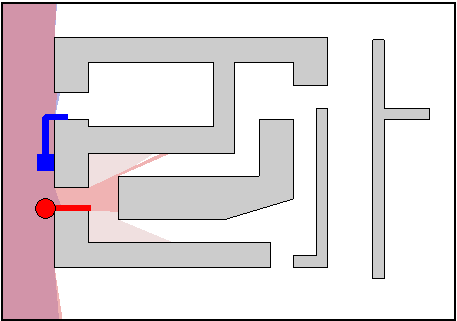}
\includegraphics[width=0.16\textwidth]{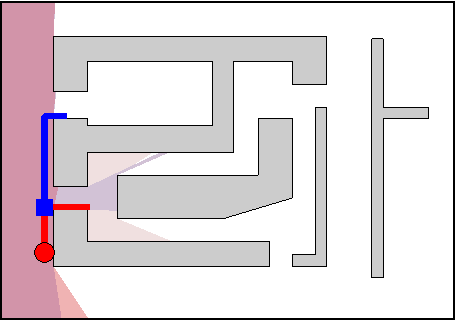}
\includegraphics[width=0.16\textwidth]{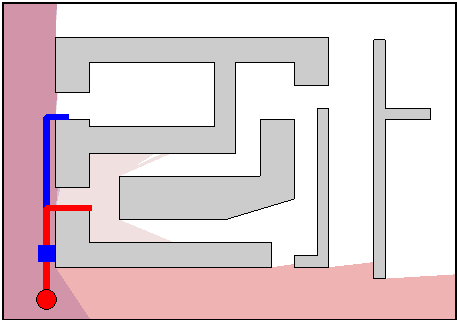}
\includegraphics[width=0.16\textwidth]{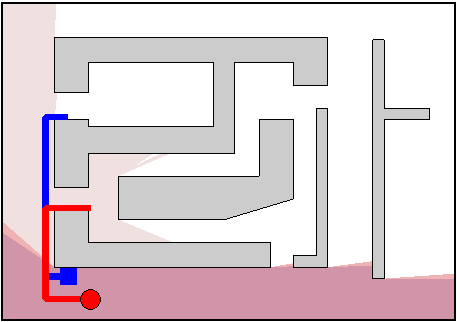}
\caption{Qualitative example (lower penalty $P=3$):  With a lower penalty, path for the agent (red) and the guard (blue) is given by MCTS for $T=10$. The agent prefers paths where it increases the area explored at the expense of being detected often. From left to right, $t=2,4,6,8,10$.
}
\label{Low_minimax}         
\end{figure*}

Both tree search methods give the same optimal solution in both cases (in general, there can be multiple optimal solutions). However, the MCTS finds the optimal solution (for $T=10$) in 40,000 iterations taking a total of approximately 50 minutes. On the other hand, the minimax tree search required approximately 10 hours to find the optimal solution. More thorough comparison is in the next subsection.

\subsection{Computational Time Comparisons}

We evaluate the computational time required to find the optimal solution by varying the time horizon $T$. Figure~\ref{Computation_cost} shows the computational time for the two search algorithms. The time horizon $T$ ranges from 1 to 5; the tree consists of 3 to 11 levels. When the time horizon $T$ is less than 3, the minimax search tree performs better than Monte-Carlo search tree. This can be attributed to the fact that Monte-Carlo search requires a certain minimum number of iterations for the estimated total reward value to converge to the actual one. When the horizon $T$ is increased, the Monte-Carlo search finds the solution faster since it does not typically require generating a full search tree. We only compare up to $T=5$ since beyond this value, we expect Monte-Carlo to be much faster than minimax search tree. Furthermore, the computational time required for finding the optimal solution for the minimax tree beyond $T=5$ is prohibitively large.

\begin{figure}
\centering{
\includegraphics[width=0.65\columnwidth]{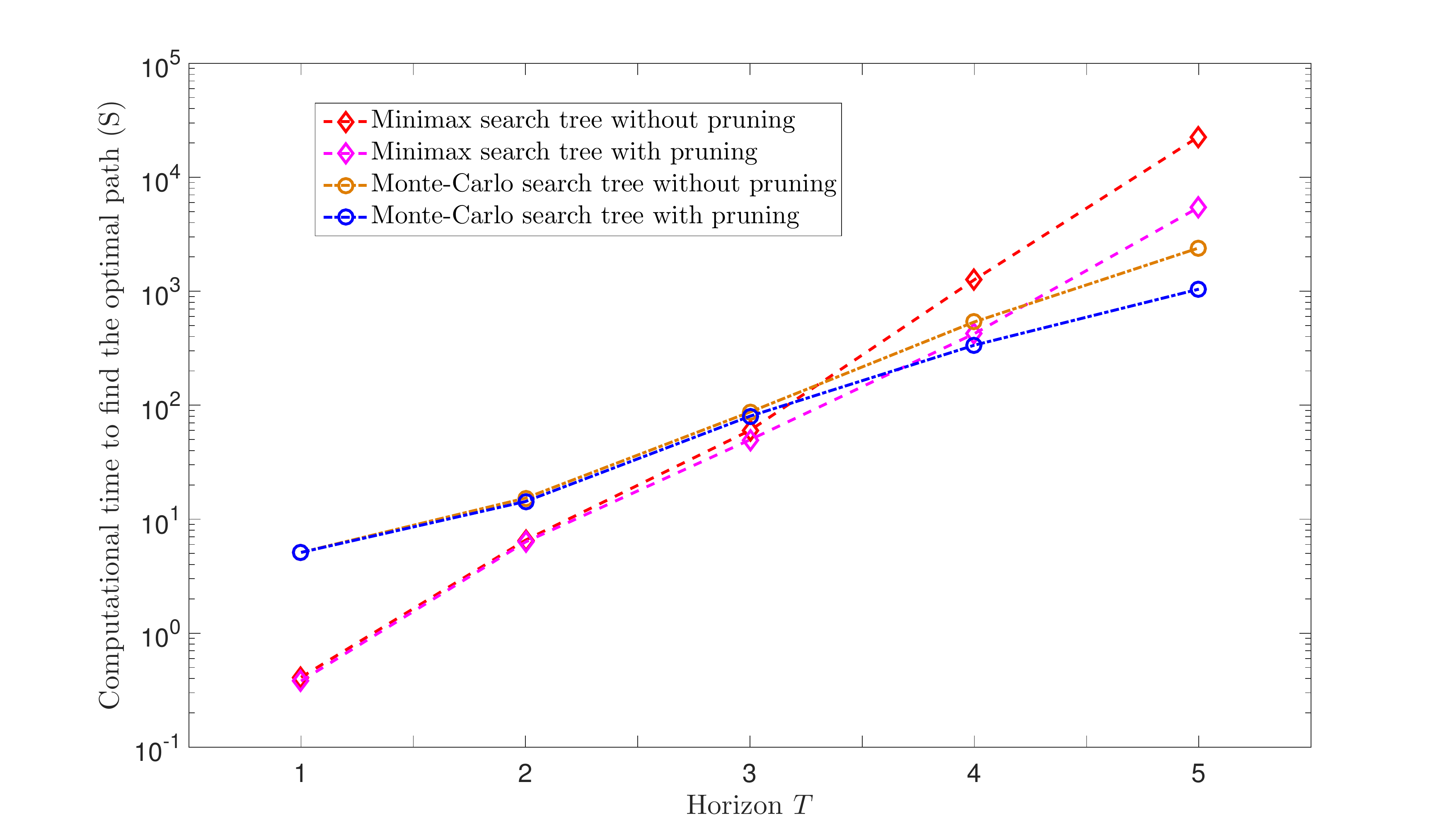}
\caption{Comparison of the time required to find the optimal solution with the minimax tree and the Monte-Carlo tree, with and without pruning. Note that the $y$ axis is in $\log$ scale. 
}
\label{Computation_cost} }    
\end{figure}


Figure~\ref{Computation_cost}, as expected, shows that the computational time with pruning is lower than that without pruning for both techniques. Next, we study this effect in more details.

\paragraph{Minimax Tree.} We show the effectiveness of the pruning algorithm by comparing the number of nodes generated by the brute force technique (no pruning) with the minimax tree with pruning.   We generate the initial position of the agent and the guard randomly. We find the optimal path for various horizons ranging from $T = 2$ to $T = 7$. Therefore, the minimax tree depth ranges from 5 to 15 (if the planning horizon is $T$, then we need a game search tree with $2T+ 1$ level). 

The efficiency of the proposed pruning algorithm is presented in Figure~\ref{fig:Prune_Compare} and Table~\ref{table:Prune_Compare}. Figure~\ref{fig:Prune_Compare} shows the combined effect of all pruning techniques by comparing it the number of nodes without pruning. Table~\ref{table:Prune_Compare} shows the individual effect of alpha-beta pruning and the combined effect of all pruning techniques.

Since the efficiency of pruning is highly dependent on the order in which the neighboring nodes are added to the tree first,  different results can be achieved by changing the order in which the children nodes are added to the minimax tree. Figure~\ref{fig:Prune_Compare}  and Table~\ref{table:Prune_Compare} compare the number of  nodes generated. Figure~\ref{fig:Prune_Compare} shows the effect of all pruning techniques and Table~\ref{table:Prune_Compare} shows the effect of individual pruning techniques.  If we enumerate all the nodes by brute force,  in the worst case, it takes $3.05\times10^8$ nodes to find the optimal path for a horizon of $T=7$. By applying the pruning algorithm, the best case only generates $2.45\times 10^5$ nodes to find the same optimal solution. 

\begin{figure}
\centering
\includegraphics[width=0.65\columnwidth]{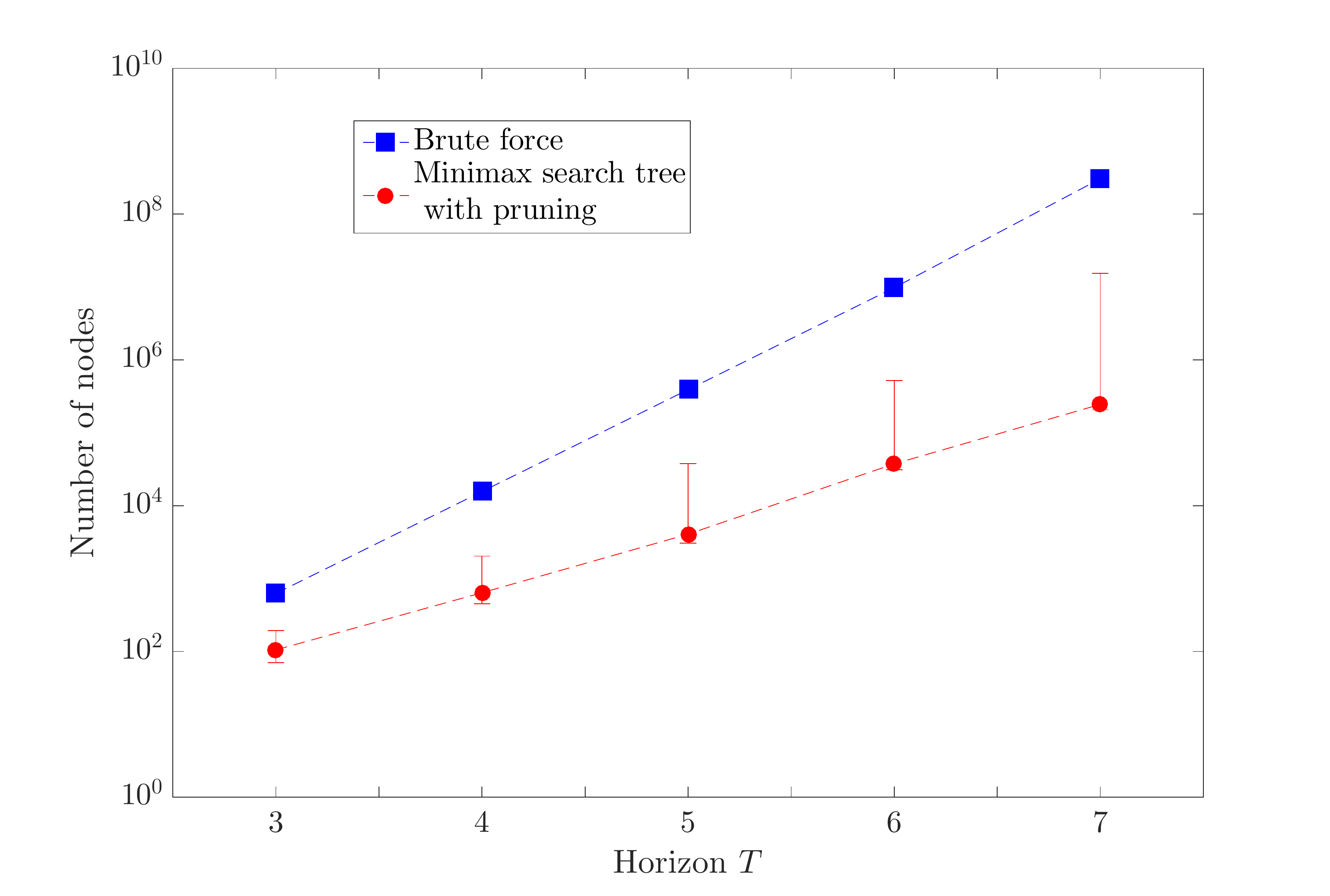}
\caption{Comparison of the number of generated nodes required to find the solution, without pruning (brute force) v.s. with pruning. From $T =3$ to $T = 7$. The error bar gives the maximum, median, and the minimum number of nodes generated (by randomly choosing one of the neighbors to expand the search tree) in 30  trials.  Note in (a), the $y$ axis is in $\log$ scale.}
\label{fig:Prune_Compare}
\end{figure}

\begin{table}[]
\caption{Comparision of the number of nodes generated by different pruning techniques,  From $T= 3$ to $T= 6$.}
\begin{tabular}{|c|c|c|c|c|c|}
\hline
\multicolumn{2}{|c|}{}                                                              & \multicolumn{4}{c|}{\textbf{Number of nodes generated}} \\ \hline
\multicolumn{2}{|c|}{Planning horizon}                                                                             & $T =3$       & $T =4$       & $T =5$      & $T =6$      \\ \hline
\multicolumn{2}{|c|}{Brute force}                                                                  & 625          & 1.56E4       & 3.90E5      & 9.76E6      \\ \hline
\multirow{3}{*}{\begin{tabular}[c]{@{}c@{}}With only\\ alpha-beta\end{tabular}}          & Maximum & 403          & 3844         & 7.08E4      & 1.70E6      \\ \cline{2-6} 
                                                                                         & Median  & 206          & 2822         & 1.80E4      & 2.46E5      \\ \cline{2-6} 
                                                                                         & Minimum & 104          & 1444         & 7860        & 1.86E5      \\ \hline
\multirow{3}{*}{\begin{tabular}[c]{@{}c@{}}With all \\ pruning\\ techniques\end{tabular}} & Maximum & 388          & 1389         & 3.3E4       & 4.81E5      \\ \cline{2-6} 
                                                                                         & Median  & 105          & 639          & 4064        & 3.74E4       \\ \cline{2-6} 
                                                                                         & Minimum & 78           & 563          & 3016        & 2.94E4        \\ \hline
\end{tabular}
\label{table:Prune_Compare}
\end{table}

\paragraph{Monte-Carlo Tree Search.} The minimax tree search method always terminates when it finds the optimal solution. On the other hand, the MCTS terminates after a pre-defined number of iterations. If this number is too low, then it is possible that the MCTS returns a sub-optimal solution.  We study the trade-off between computational time and the number of iterations for the MCTS.

\begin{figure}
\centering{
\includegraphics[width=0.65\columnwidth]{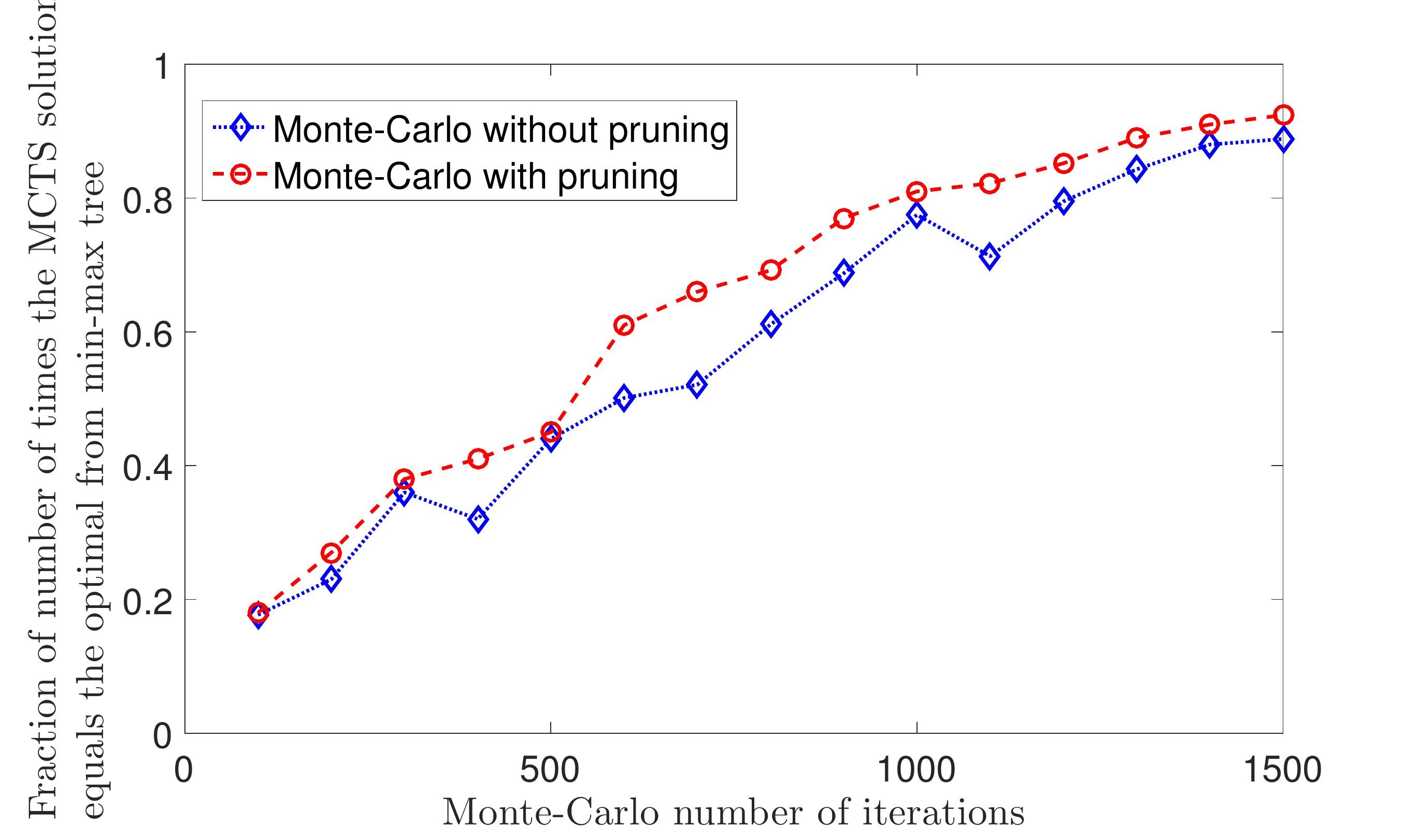}
\caption{Effect of increasing the number of iterations in MCTS, with and without pruning, on the the likelihood of finding the optimal solution. The $y$--axis shows the fraction of the number of trials (out of 50 trials) MCTS was able to find the optimal solution given by the minimax tree for $T=3$.}
\label{rate_best_solution_found} }    
\end{figure}

Figure~\ref{rate_best_solution_found} shows the fraction of the times we find the optimal solution as a function of the number of iterations when $T=3$ in a $10\times 10$ grid map. We first find the optimal solution using a minimax tree. Then, we run the MCTS for a fixed number of iterations and verify if the best solution found has the same value as the optimal. The X-axis in this figure is the number of iterations in MCTS. 

We make the following observations from Figure~\ref{rate_best_solution_found}: (1) The proposed pruning strategy increases the (empirical) likelihood of finding the optimal solution in the same number of iterations; and (2) The probability of finding the optimal solution grows as the number of iterations grows.

\begin{figure}
\centering{
\includegraphics[width=0.65\columnwidth]{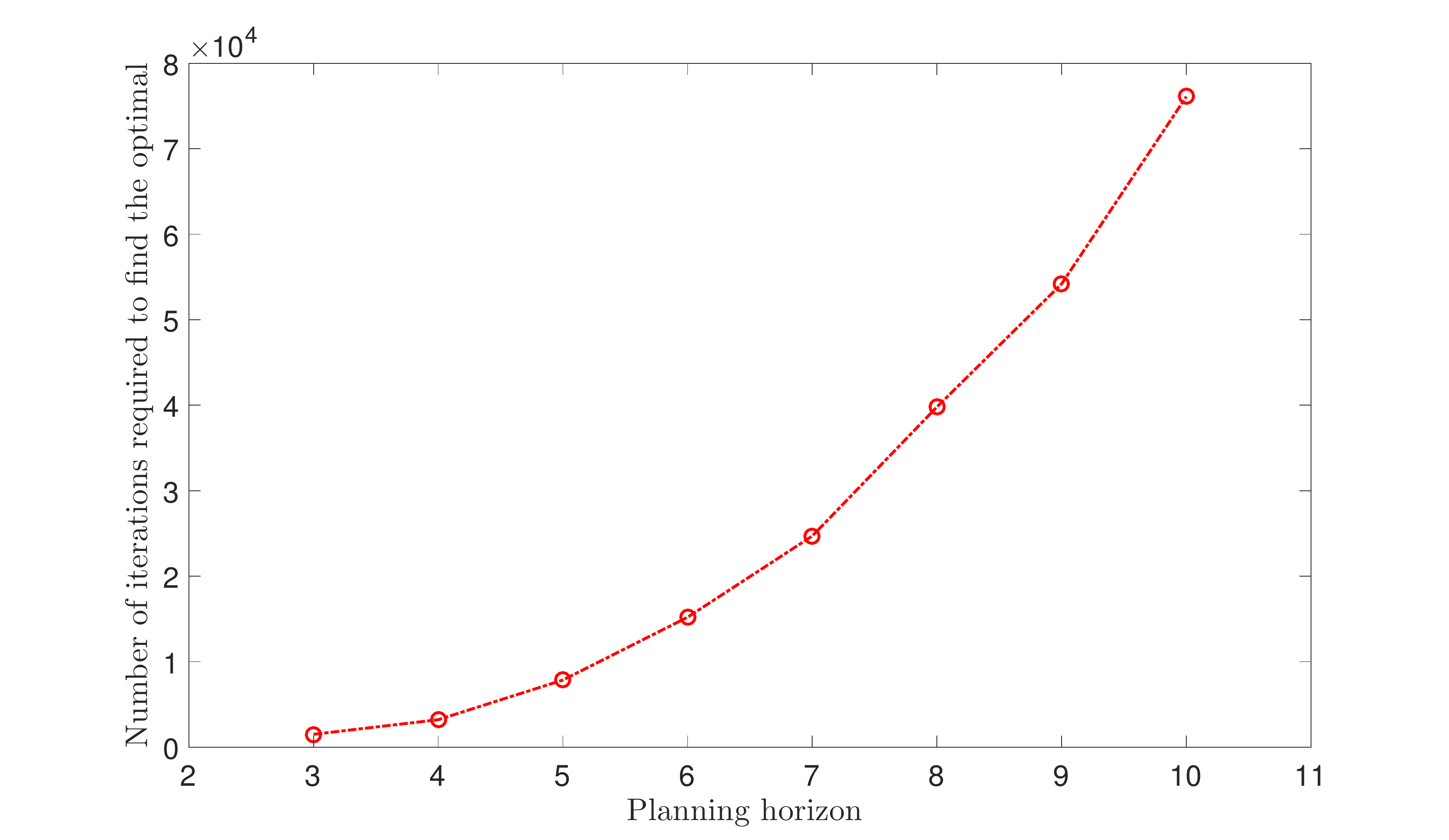}
\caption{Effect of the planning horizon on the number of iterations required to find the optimal solution for MCTS with pruning. }
\label{iterations_over_the_solution_quality} }    
\end{figure}
The number of iterations required to find the optimal solution also depends on the planning horizon. Figure~\ref{iterations_over_the_solution_quality} shows the effect of the planning horizon over the number of iterations required to find the optimal solution. Note that even though the likelihood of finding an optimal solution increases with more iteration times in general, it is always possible that only a suboptimal is found due to ``overfitting'' caused by the UCB selection rule. Therefore, for the following simulations, we run the MCTS multiple times and find out how often we find the optimal solution within a given number of iterations. If we find the optimal solution 80\% or more times, we consider it as success. We find that the number of iterations required to find success 80\% or more times increases exponentially as we vary the planning horizon.

\section{Conclusion and Discussion} \label{sec:con}
We introduce a new problem of maximizing visibility and minimizing detectability in an environment with an adversarial guard. The problem can be solved using a minimax and the MCTS techniques to obtain an optimal strategy for the agent. Our main contribution is a set of pruning techniques that reduce the size of the search tree while still guaranteeing optimality.


Despite the promising reduction in the game tree, the method can still be time consuming when the planning horizon increases or if the environment becomes large and/or complex. Our immediate work is to further reduce the computational effort using MCTS with macro-actions~\cite{lim2011monte}, and by exploiting the underlying geometry of the environment.

\section*{LEGAL}
DISTRIBUTION A. Approved for public release: distribution unlimited. This research was supported in part by the Automotive Research Center (ARC) at the University of Michigan, with funding and support by the Department of Defense under Contract No. W56HZV-14-2-0001.

\bibliographystyle{IEEEtran}
\bibliography{IEEEabrv,main,refs}

\end{document}